\documentclass{article}

\usepackage[main, final, nonatbib]{neurips_2026}

\usepackage[utf8]{inputenc} %
\usepackage[T1]{fontenc}    %
\usepackage{hyperref}       %
\usepackage{url}            %
\usepackage{booktabs}       %
\usepackage{amsfonts}       %
\usepackage{nicefrac}       %
\usepackage{microtype}      %
\usepackage{xcolor}         %

\usepackage{float}
\newtheorem{theorem}{Theorem}

\usepackage{url}
\usepackage{footnote}
\usepackage[bottom]{footmisc}
\makeatletter
\g@addto@macro{\UrlBreaks}{\UrlOrds}
\makeatother
\usepackage{makecell}

\usepackage{epsfig}
\usepackage{graphicx}
\usepackage{amsmath}
\usepackage{amssymb}
\usepackage{bbm}
\usepackage{soul}

\usepackage{float}
\usepackage{pifont}%

\newcommand{\cmark}{\ding{51}}%

\usepackage{xspace}

\usepackage{multirow}
\usepackage{color, colortbl}
\usepackage{booktabs}
\usepackage{algorithm, algpseudocode}
\usepackage{tcolorbox}
\usepackage{tabularx}
\usepackage{arydshln}
\usepackage{caption}

\usepackage{dblfloatfix}

\usepackage{tikz}
\usetikzlibrary{bayesnet}
\usetikzlibrary{arrows}
\usepackage{float}

\usepackage{wrapfig}

\definecolor{main}{HTML}{5989cf}    %
\definecolor{sub}{HTML}{cde4ff}     %
\definecolor{greyborder}{HTML}{808080}
\definecolor{lightgrey}{HTML}{D3D3D3}

\tcbset{
    sharp corners,
    colback = white,
    before skip = 0.2cm,    %
    after skip = 0.5cm      %
}                           %

\newtcolorbox{boxA}{
    fontupper = \bf,
    boxrule = 1.5pt,
    colframe = black %
}

\newtcolorbox{boxB}{
    fontupper = \bf\color{main}, %
    boxrule = 1.5pt,
    colframe = main,
    rounded corners,
    arc = 5pt   %
}

\newtcolorbox{boxC}{
    colback = sub, %
    boxrule = 0pt  %
}

\newtcolorbox{boxD}{
    colback = lightgrey, 
    colframe = greyborder, 
    boxrule = 0pt, 
    toprule = 3pt, %
    bottomrule = 3pt %
}

\newtcolorbox{boxE}{
    enhanced, %
    boxrule = 0pt, %
    borderline = {0.75pt}{0pt}{main}, %
    borderline = {0.75pt}{2pt}{sub} %
}

\newtcolorbox{boxF}{
    colback = sub,
    enhanced,
    boxrule = 1.5pt, 
    colframe = white, %
    borderline = {1.5pt}{0pt}{main, dashed} %
}

\newtcolorbox{boxG}{
    enhanced,
    boxrule = 0pt,
    colback = sub,
    borderline west = {1pt}{0pt}{main}, 
    borderline west = {0.75pt}{2pt}{main}, 
    borderline east = {1pt}{0pt}{main}, 
    borderline east = {0.75pt}{2pt}{main}
}

\newtcolorbox{boxH}{
    colback = sub, 
    colframe = main, 
    boxrule = 0pt, 
    leftrule = 6pt %
}

\newtcolorbox{boxI}{
    colback = sub, 
    colframe = main, 
    boxrule = 0pt, 
    toprule = 6pt %
}

\newtcolorbox{boxJ}{
    sharpish corners, %
    colback = sub, 
    colframe = main, 
    boxrule = 0pt, 
    toprule = 4.5pt, %
    enhanced,
    fuzzy shadow = {0pt}{-2pt}{-0.5pt}{0.5pt}{black!35} %
}

\newtcolorbox{boxK}{
    sharpish corners, %
    boxrule = 0pt,
    toprule = 4.5pt, %
    enhanced,
    fuzzy shadow = {0pt}{-2pt}{-0.5pt}{0.5pt}{black!35} %
}

\newtcolorbox{boxL}{
    fontupper = \color{main},
    rounded corners,
    arc = 6pt,
    colback = sub, 
    colframe = main!50, 
    boxrule = 0pt, 
    bottomrule = 4.5pt 
}

\newtcolorbox{boxM}{
    fontupper = \color{white},
    rounded corners,
    arc = 6pt,
    colback = main!80, 
    colframe = main, 
    boxrule = 0pt, 
    bottomrule = 4.5pt,
    enhanced,
    fuzzy shadow = {0pt}{-3pt}{-0.5pt}{0.5pt}{black!35}
}

\usepackage[colorinlistoftodos]{todonotes}

\newcounter{mycomment}

\title{
Agentic Multi-Turn Reasoning: A Fairness Approach
}

\author{%
Thanh-Dat Truong$^{1}$, Sankalp Pandey$^{1}$, Hugh Churchill$^{2}$, Jackson Cothren$^{3}$\\\textbf{Marios Savvides}$^{4}$, \textbf{Khoa Luu}$^{1}$\\
$^{1}$CVIU Lab, University of Arkansas, USA \quad
$^{2}$Dep. of Physics, University of Arkansas, USA \\
$^{3}$Dep. of Geosciences, University of Arkansas, USA \quad $^{4}$Carnegie Mellon University, USA\\
\tt\small \{tt032,  sankalpp, hchurch, jcothre, khoaluu\}@uark.edu, \tt\small marioss@andrew.cmu.edu\\
\tt\small \url{https://uark-cviu.github.io/projects/Phi-MPO}
}

\begin{document}

\maketitle

\begin{abstract}

Recent advances in Large Language Models (LLMs) have enabled agentic systems capable of solving complex tasks through multi-turn planning, tool use, verification, and memory updates. However, learning agentic systems remains difficult due to two fundamental challenges, i.e., (1) long-horizon credit assignment, where supervision is available only at the final outcome, and (2) imbalanced data distributions, where dominant data patterns bias optimization and weaken adaptation to rare but informative reasoning behaviors. In this paper, we propose Fair Multi-Level Preference Optimization
(Fair-MPO or $\Phi$-MPO), a new preference optimization framework for agentic learning. We first show that Multi-Level Preference Optimization provides a principled and more computationally efficient framework for long-horizon reasoning. Then, we introduce a Fair Multi-Level Objective that addresses imbalance in agentic learning. We provide a comprehensive theoretical analysis demonstrating that our approach addresses both long-horizon reasoning and data imbalance. Our experiments on agentic reasoning benchmarks demonstrate that our approach achieves State-of-the-Art (SOTA) performance. 

\end{abstract}
\begin{figure}[H]
    \centering
    \vspace{-5mm}
    \includegraphics[width=0.8\linewidth]{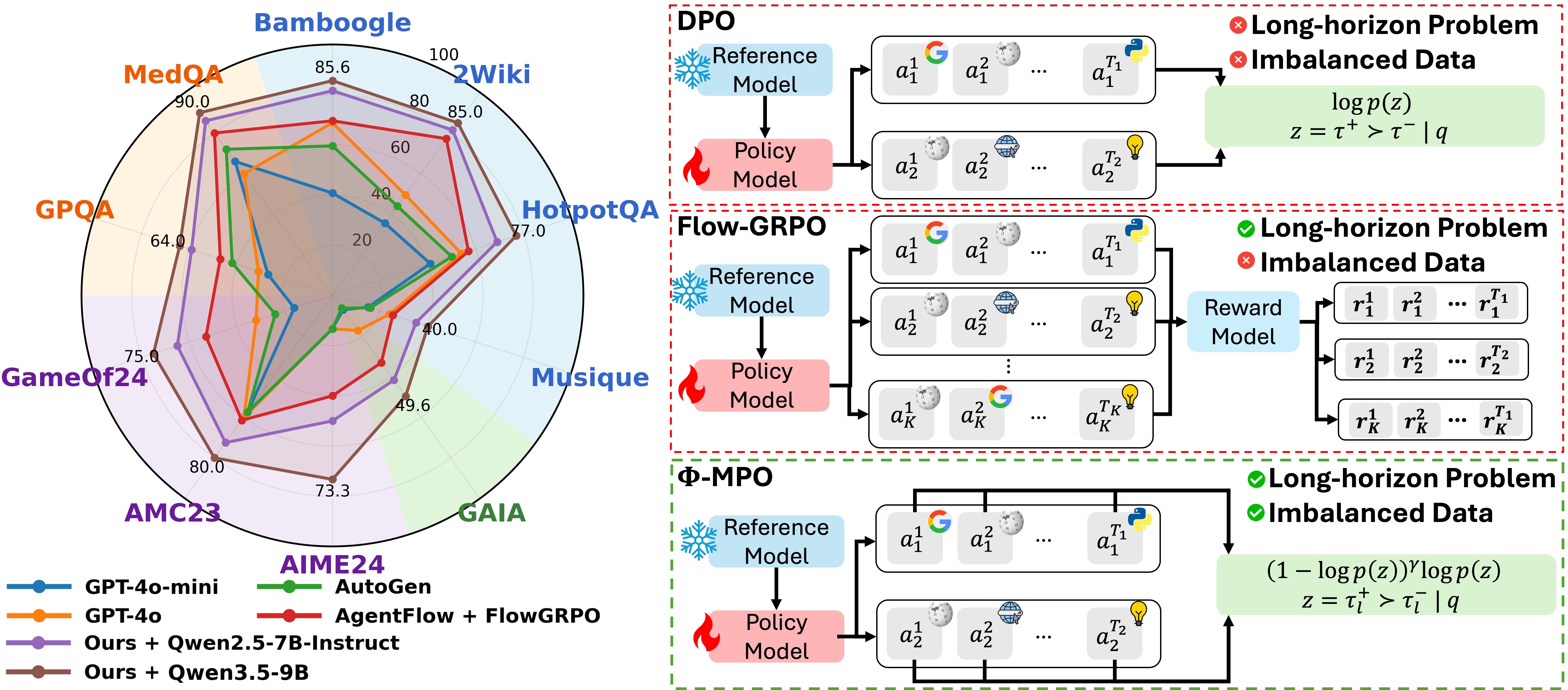}
    \caption{\textbf{Overview of the proposed $\Phi$-MPO approach.} Left: $\Phi$-MPO consistently outperforms GPT-4o, AgentFlow + FlowGRPO across four benchmarks. Right: Compared with DPO and Flow-GRPO, $\Phi$-MPO addresses both long-horizon reasoning and imbalanced problems. 
    }
    \label{fig:highlight}
    \vspace{-6mm}
\end{figure}

\section{Introduction}

The success of Large Language Models (LLMs) \cite{achiam2023gpt, comanici2025gemini, truong2026directed} has motivated a new class of \emph{agentic} systems that reason over multiple steps, interact with external tools, maintain intermediate memory, and refine decisions based on feedback \cite{li2025search, jin2025search, li2025flow, guo2025deepseek}.
Unlike single-turn generation, agentic reasoning is a multi-turn process in which the model must determine \emph{when} to plan, \emph{when} to call tools, \emph{how} to update memory, and \emph{how} to recover from mistakes, with each action shaping subsequent reasoning states and final results \cite{jin2025search, li2025flow, song2025r1, jiang2026xskill}.
However, most existing agentic systems remain training-free, relying on prompting heuristics, handcrafted coordination rules, or static orchestration strategies \cite{wu2024autogen, hong2023metagpt}, which limits their adaptability to new environments, tools, and task distributions. Therefore, reliable and scalable multi-turn reasoning requires a \textbf{learning paradigm for agentic systems} that optimizes reasoning trajectories, tool usage, and module coordination based on interaction outcomes (Figure~\ref{fig:highlight}).

\begin{wrapfigure}{r}{0.5\linewidth}
    \centering
    \vspace{-5mm}
    \includegraphics[width=1.0\linewidth]{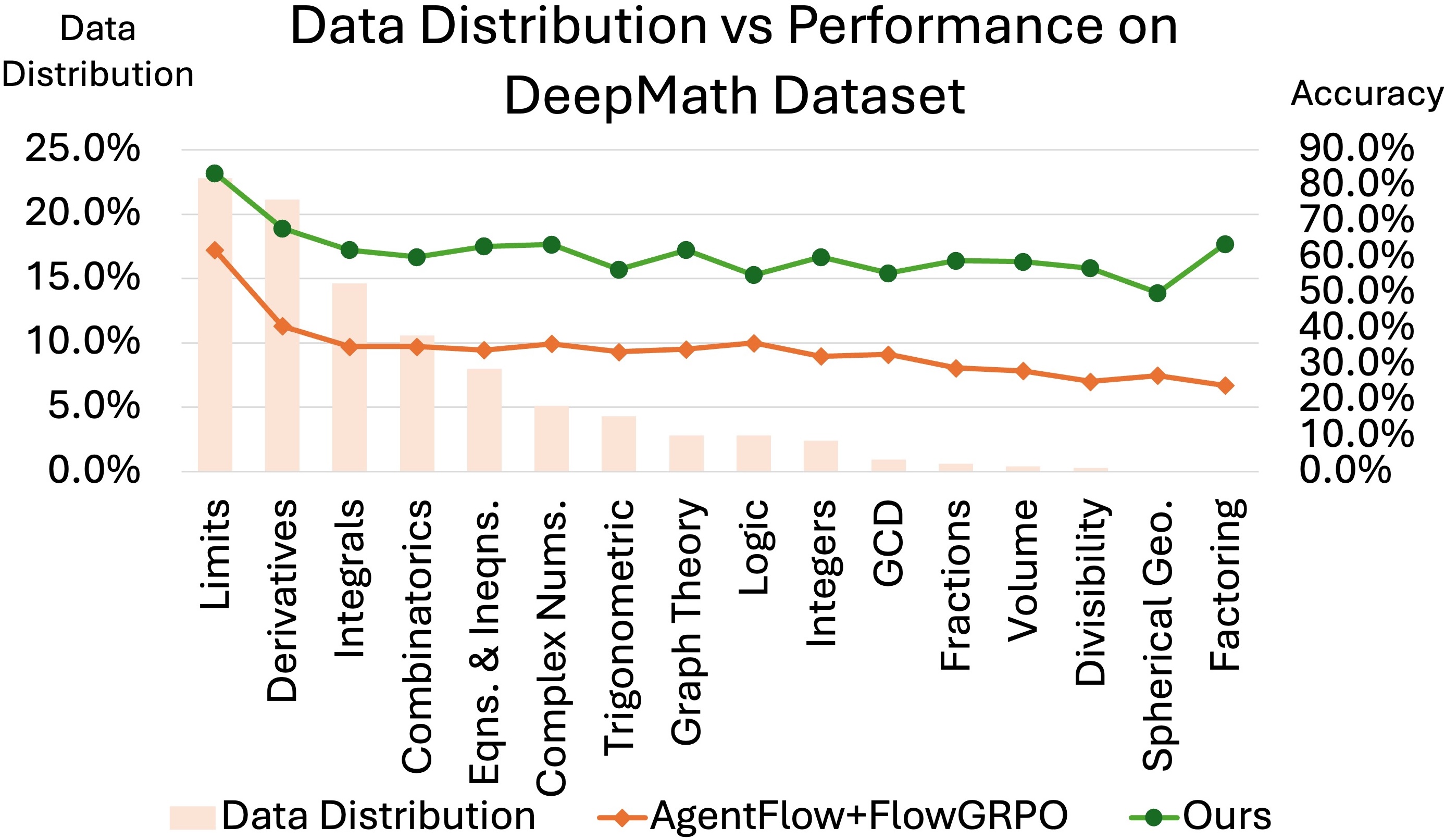}
    \vspace{-6mm}
    \caption{The Influence of Imbalance Data. 
    }
    \label{fig:deepmath-data-dist}
    \vspace{-4mm}
\end{wrapfigure}
In learning agentic systems, two fundamental challenges arise: the \textbf{long-horizon problem} and the \textbf{imbalance problem}. The long-horizon problem arises because agentic reasoning unfolds over many interdependent steps, while supervision is typically available only at the final outcome \cite{li2025flow}. Thus, early mistakes can propagate through planning, tool use, and memory updates, making it difficult to assign credit to intermediate decisions. Meanwhile, agentic learning often relies on imbalanced reasoning data, e.g., DeepMath \cite{he2025deepmath}, where topics, task types, difficulty levels, and reasoning patterns appear at highly uneven frequencies. As shown in Figure~\ref{fig:deepmath-data-dist}, this imbalance directly affects model performance. These challenges are amplified by the structured nature of agentic rollouts. Unlike single-output generation, an agentic rollout is a trajectory $\tau=(state_1,action_1,\dots,state_T,action_T)$, where each state depends on prior reasoning, tool feedback, and memory updates. Thus, data-level imbalance propagates into trajectory-level imbalance: dominant topics and familiar tool-use patterns produce many redundant trajectories, while successful trajectories for rare or difficult cases remain scarce. This skew is further compounded by error propagation, e.g., an incorrect search query may retrieve irrelevant evidence, mislead later reasoning, and ultimately produce an incorrect answer. Therefore, a principled learning framework must address both long-horizon credit assignment and the propagation of imbalance from data to trajectories.

Recent studies have made important progress, but remain inadequate under long horizons and imbalanced data. Outcome-based reinforcement learning (RL) methods \cite{shao2024deepseekmath, yu2025dapo, song2025r1} and in-the-flow optimization frameworks \cite{li2025flow} optimize trajectory generation from final rewards, yet often suffer from high variance, unstable optimization, and weak credit assignment when rewards are sparse. 
Meanwhile, Proximal Policy Optimization (PPO) \cite{schulman2017proximal}, Group Relative Policy Optimization (GRPO) \cite{shao2024deepseekmath}, or Flow-GRPO \cite{li2025flow} improve policy learning from interaction trajectories, with GRPO using group-relative normalization and Flow-GRPO propagating final rewards across intermediate rollout steps. However, credit assignment remains coarse because the same outcome reward is shared across many decisions that differ in causal importance. Moreover, when successful rollouts are rare, optimization can be dominated by unsuccessful or low-quality trajectories, leading to biased updates and weak recovery from early mistakes.
Direct Preference Optimization (DPO) provides a more stable alternative, learning from pairwise comparisons without explicit reward modeling or policy-gradient updates \cite{rafailov2023direct}. However, standard DPO is designed for single-output preference and does not explicitly model the sequential structure of agentic trajectories. 
Since DPO relies on informative and reasonably balanced preference pairs, it is also vulnerable to the highly skewed distributions produced by agentic rollouts. Thus, neither existing RL-based methods nor standard DPO fully addresses the dual challenge of long horizon and imbalance in agentic learning.

\noindent
\textbf{Contributions.}
In this work, we propose \textbf{Fair Multi-Level Preference Optimization} (\textbf{Fair-MPO} or \textbf{$\Phi$-MPO}) for learning multi-turn agentic reasoning with long-horizon feedback and an imbalanced data distribution. Our contributions can be summarized as follows. First, we introduce a new \textbf{Multi-level Preference Learning} (MPO) framework for agentic systems that models each rollout as a structured reasoning trajectory and optimizes the policy through both trajectory-level and state-level supervision. Second, we identify the \textbf{imbalance problem in agentic learning} followed by developing a new \textbf{\textbf{$\Phi$-MPO} learning objective} that improves robustness to skewed distributions. 
Third, we provide a theoretical analysis from two complementary perspectives: (1) we first show that DPO shares the same pairwise optimization direction as GRPO, yielding a more computationally efficient alternative. and (2) we then show that our fair multi-level learning objective mitigates optimization bias under imbalanced data.
Finally, extensive experiments and ablation studies on standard benchmarks demonstrate that our approach improves final reasoning accuracy over prior methods. Our results establish $\Phi$-MPO as an effective learning framework for agentic systems in dynamic multi-turn environments.

\section{Related Work}

\noindent
\textbf{Agentic Learning.} Recent studies formulates tool-integrated reasoning as a RL problem, where a LLM learns to interleave reasoning with external actions.
Most methods adopt a monolithic policy that decides when to reason, when to invoke tools, and how to incorporate feedback. 
This paradigm has shown strong results in mathematical reasoning \cite{mai2025agent,xue2025simpletir,feng2025retool,li2025torl, jiang2026xskill, truong2026directed, truong2026mango, truong2025insect}, web-based question answering \cite{chen2025learning,jin2025search,song2025r1,li2025search,sun2025zerosearch}, and multi-tool environments through data synthesis, unified training pipelines, and reward design \cite{dong2025tool,jiang2025verltool,qian2025toolrl,zhang2025nemotron}. 
However, as task complexity and planning horizon increase, monolithic optimization faces severe long-horizon credit assignment, since final outcomes must be attributed to many intermediate reasoning steps, routing decisions, and tool calls \cite{zeng2025reinforcing,wang2025stepsearch}. 
This can lead to unstable optimization, 
including inefficient tool use \cite{wang2025acting,qian2025smart}, weak personalization \cite{cheng2025toolspectrum}, and poor alignment with user preferences \cite{huang2025ttpa}.
Recent agentic systems offer an alternative by decomposing reasoning into specialized modules or agents for planning, execution, verification, and coordination. 
Many systems are training-free and orchestrate pretrained LLMs through prompting, role assignment, or handcrafted control logic \cite{wu2024autogen,hong2023metagpt,lu2025octotools}. 
While modularity improves flexibility, static orchestration limits adaptation because collaboration strategies are prescribed rather than learned from interaction. 
Recent methods train agentic systems for improved coordination and tool use \cite{deng2025pe,liao2025marft}, but often rely on offline supervised fine-tuning or preference optimization over static datasets \cite{motwani2024malt,park2025maporl}. 
Thus, training remains disconnected from live multi-turn agent-tool dynamics, limiting adaptation to evolving observations, recovery from early mistakes, and plan revision over long horizons.

\noindent
\textbf{Preference Learning and Policy Optimization.}
Preference learning and policy optimization are widely used to align LLMs with human or task-level feedback \cite{ziegler2019fine,christiano2017deep}. 
Standard Reinforcement Learning from Human Feedback (RLHF) trains a reward model from ranked outputs and then optimizes the policy with PPO \cite{schulman2017proximal,ouyang2022training}. 
DPO \cite{rafailov2023direct, truong2026phi, liu2025focalpo} instead directly optimizes preference pairs, avoiding explicit reward modeling.
For reasoning, GRPO \cite{shao2024deepseekmath} uses group-relative comparisons to stabilize policy optimization, while Flow-GRPO \cite{li2025flow} extends on-policy optimization to long-horizon agentic rollouts with sparse trajectory-level rewards. 
Despite this progress, RL methods still provide coarse supervision under delayed feedback, and standard preference optimization remains designed for single-output comparisons rather than structured, imbalanced, and multi-step agentic trajectories.

\section{The Proposed Fair Multi-level Preference Optimization ($\Phi$-MPO) Approach}

In this section, we propose a new learning approach that enables the optimization of agentic systems within a multi-turn reasoning loop.
Formally, given an input query $q$ and a predefined toolset $\mathcal{K}$, our goal is to learn a policy that composes tool-usage actions to solve the query through iterative interaction.
Following \cite{li2025flow, jin2025search}, the agentic reasoning is decomposed into four modules: 
$\textbf{Action Planner } \mathcal{P}$, 
$\textbf{Tool Executor } \mathcal{E}$, 
$\textbf{Execution Verifier } \mathcal{V}$, 
and $\textbf{Solution Generator } \mathcal{G}$. 
These modules communicate through a shared evolving memory $\mathcal{M}$, which stores intermediate reasoning traces, tool outputs, and verification feedback. 
The overall process is formulated as a multi-turn Markov Decision Process (MDP). 
Let $\mathcal{M}^t$ denote the memory before turn $t$, with $\mathcal{M}^1$ initialized from $q$. 
At each turn, the planner $\mathcal{P}$ is modeled as a trainable policy $\pi_\theta$ that samples an action $a^t \sim \pi_\theta(a^t \mid q,\mathcal{K},\mathcal{M}^t)$, 
where $a^t$ specifies the current reasoning decision, including sub-goal generation, tool selection, and memory-conditioned retrieval. 
Conditioned on $a^t$, the executor invokes the selected tool and returns an observation $e^t \sim \mathcal{E}(e^t \mid a^t,\mathcal{K})$. 
The verifier then produces a binary signal $v^t \sim \mathcal{V}(v^t \mid q,e^t,\mathcal{M}^t)$ indicating whether the accumulated evidence is sufficient to generate the final answer. 
If $v^t=0$, the memory is updated as $\mathcal{M}^{t+1} = f_{\mathrm{mem}}(\mathcal{M}^t,a^t,e^t,v^t)$, 
providing an explicit and controllable representation of intermediate reasoning states. 
The process continues until $v^t=1$ or the maximum number of turns is reached. 
At the final turn $T$, the solution genera tor produces the output $o \sim \mathcal{G}(o \mid q,\mathcal{M}^T)$, 
conditioned on both the query and the accumulated reasoning evidence. Then, 
the multi-turn reasoning process defines a structured trajectory $\tau = \{(a^t, e^t, v^t)\}_{t=1}^{T}$, which records the sequence of planning decisions, tool interactions, and verification signals as in Eqn. \eqref{eqn:joint_policy_rewrite}.
\begin{equation}
\label{eqn:joint_policy_rewrite}
\footnotesize
\begin{aligned}
\pi_\theta\!\left(\tau, o \mid q\right) = \Bigg[ \prod_{t=1}^{T} \pi_\theta(a^t \mid q, K, \mathcal{M}^t) \; \mathcal{E}(e^t \mid a^t, K) \; \mathcal{V}(v^t \mid q, e^t, \mathcal{M}^t) \Bigg] \; \mathcal{G}(o \mid q, \mathcal{M}^t).
\end{aligned}
\end{equation}
Similar to \cite{li2025flow}, Eqn. \eqref{eqn:joint_policy_rewrite} decomposes complex reasoning into a sequence of observable transitions, allowing the reasoning trajectory to be explicitly represented through structured intermediate states. 
The planner $\pi_\theta$ is optimized within the reasoning loop, enabling direct improvement of intermediate decisions and more fine-grained credit assignment across multi-turn, tool-integrated trajectories.

\noindent
\textbf{Agentic Learning.} The Action Planner policy $\pi_\theta$ can be optimized via Reinforcement Learning from Human Feedback (RLHF). In particular, let $R(\tau)$ be the reward model to evaluate the correctness and efficiency of a complete trajectory $\tau$. The planner can be optimized as follows:
\begin{equation}
\footnotesize
\label{eqn:rlhf_constrained}
\begin{split}
\pi_{\theta}^\star \;=\; \arg\!\max_{\pi_{\theta}}\; \mathbb{E}_{q \sim \mathcal{Q}}\mathbb{E}_{\tau \sim \pi_\theta(\tau \mid q)| }\big[ R(\tau)\big] 
\quad\text{s.t.}\quad
D_{\mathrm{KL}}\!\big(\pi_t(\cdot\mid q)\,\big\|\,\pi_{ref}(\cdot\mid q)\big)\le \delta,
\end{split}
\end{equation}
where $\pi_{ref}$ is the reference policy, $D_\textrm{KL}(\pi_{\theta} \| \pi_{ref})$ is the KL divergence to measure the difference between two policies, and $\tau$ is a trajectory rollout.
In multi-turn reasoning, the contribution of each intermediate action $a^t$ is difficult to observe directly, since its effect may only appear after several subsequent steps.
Thus, defining step-wise rewards can introduce noisy or inconsistent supervision.
Following \cite{li2025flow}, we instead adopt a trajectory-level reward, where all actions in a rollout share a common signal determined by the correctness of the final solution, i.e., $r = R(a^t) = \bar{R}(o, q, y^*)$ ($\forall t = 1,\dots,T$), 
where $y^*$ is the ground-truth answer, and $\bar{R}(o,q,y^*) \in \{0,1\}$ evaluates whether the final prediction $o$ is correct. Similar to \cite{li2025flow}, we adopt the LLM-as-judge protocol for the reward function $\bar{R}(\cdot)$.
This approach ensures a global correctness signal throughout the reasoning trajectory, encouraging each intermediate decision $a^t$ to contribute to producing a correct final solution.

\noindent
\textbf{Challenges and Limitations of Prior Agentic Learning.}
Technically, multi-turn agentic learning must handle both long-horizon credit assignment and imbalanced data distributions. 
Given a trajectory $\tau=\{a^t\}_{t=1}^T$, the final output depends on the accumulated memory $\mathcal{M}^T$, so the effect of an early action may only appear after several planning, execution, and memory-update steps. 
Meanwhile, an imbalance in topics, tasks, and reasoning patterns propagates into rollouts, causing dominant trajectories or tool-use behaviors to receive disproportionate gradient updates.
While Eqn.~\eqref{eqn:rlhf_constrained} can be optimized with PPO \cite{schulman2017proximal}, PPO requires explicit reward modeling and value estimation over long trajectories, which can be biased and unstable under final-outcome supervision. 
GRPO \cite{li2025flow} mitigates value estimation via group-relative comparison, but still requires expensive on-policy sampling over multiple long rollouts involving tool calls, verification, and memory updates. 
These limitations motivate a Direct Preference Optimization approach that avoids explicit reward and value modeling while providing more stable learning under long-horizon dependency and imbalanced data.

\subsection{Agentic Learning via Preference Optimization}

The constrained optimization problem in Eqn.~\eqref{eqn:rlhf_constrained} can be rewritten using the Lagrangian \cite{rafailov2023direct} form as:
\begin{equation}
\label{eqn:rlhf_agentic_beta}
\footnotesize
\begin{split}
\pi_{\theta}^{\star} = \arg\!\max_{\pi_{\theta}} \; \mathbb{E}_{q \sim \mathcal{Q}} \mathbb{E}_{\tau \sim \pi_{\theta}(\tau \mid q)} \big[ R(\tau) \big] - \beta D_{\mathrm{KL}} \!\Big( \pi_{\theta}(\cdot \mid q) \,\Big\|\, \pi_{\mathrm{ref}}(\cdot \mid q) \Big), 
\end{split}
\end{equation}
where 
$\beta$ is the Lagrangian multiplier that controls the deviation of the current policy from the reference model. 
Then, the optimal policy of Eqn.~\eqref{eqn:rlhf_agentic_beta} \cite{rafailov2023direct} admits the following closed form:
\begin{equation}
\label{eqn:agentic_boltzmann}
\footnotesize
\begin{split}
\pi_{\theta}^{\star}(\tau \mid q) &= \frac{ \pi_{\mathrm{ref}}(\tau \mid q) \exp\!\left(\tfrac{1}{\beta}R(\tau)\right) }{ Z(q) }, \qquad
R(\tau) = \beta \log \frac{\pi_{\theta}^{\star}(\tau \mid q)} {\pi_{\mathrm{ref}}(\tau \mid q)} + \beta \log Z(q), 
\end{split}
\end{equation}
where $Z(q)=\sum_{\tau}\pi_{\mathrm{ref}}(\tau \mid q)\exp\!\left(\frac{1}{\beta}R(\tau)\right)$ is the partition function.

\noindent
\textbf{Pairwise Trajectory Preferences.}
For each query $q$, we denote by $\tau^{+}$ a preferred trajectory that yields a more effective reasoning process or a correct final solution, and by $\tau^{-}$ a dispreferred trajectory that produces inferior reasoning or an incorrect solution.
Then, the preference of $\tau^{+}$ over $\tau^{-}$ conditioned on query $q$ can be modeled using the Bradley-Terry model, i.e., $p(\tau^{+} \succ \tau^{-} \mid q) = \sigma\!\Big( \beta\big[ R(\tau^{+}) - R(\tau^{-}) \big] \Big)$, 
where $\sigma(u)$ is the sigmoid function. Accordingly, maximizing the conditional likelihood over pairwise trajectory preferences can be written as the following objective:
\begin{equation}
\label{eqn:agentic_pref_loss}
\footnotesize
\begin{split}
\mathcal{L}_{\mathrm{pref}}
=
\mathbb{E}_{(q,\tau^{+},\tau^{-})}
\Big[
-\log
\sigma\!\Big(
\beta\big[
R(\tau^{+}) - R(\tau^{-})
\big]
\Big)
\Big].
\end{split}
\end{equation}

\noindent
\textbf{Agentic Learning Paradigm with Trajectory-level DPO.} 
Under the optimality condition in Eqn.~\eqref{eqn:agentic_boltzmann}, the reward difference in Eqn.~\eqref{eqn:agentic_pref_loss} can be rewritten in terms of trajectory-level policy log-ratios. Therefore, the preference objective 
can be reformulated as the following \textbf{Trajectory-level DPO} loss:
\begin{equation}
\label{eqn:agentic_dpo_loss}
\footnotesize
\begin{split}
&\mathcal{L}_{\mathrm{DPO}}(\pi_{\theta}, \pi_{\mathrm{ref}})
=
-\mathbb{E}_{q,\tau^{+},\tau^{-}}
\Bigg[
\log \sigma \Bigg(
\beta
\log
\frac{\pi_{\theta}(\tau^{+}\mid q)}
{\pi_{\mathrm{ref}}(\tau^{+}\mid q)}
-
\beta
\log
\frac{\pi_{\theta}(\tau^{-}\mid q)}
{\pi_{\mathrm{ref}}(\tau^{-}\mid q)}
\Bigg)
\Bigg].
\end{split}
\end{equation}

Since each trajectory $\tau^+$ or $\tau^-$ is generated through a multi-turn reasoning process, its likelihood can be factorized over the sequence of intermediate actions as Eqn.~\eqref{eqn:joint_policy_rewrite}. 

\noindent
\textbf{Interpretation.}
The Trajectory-level DPO optimizes the planner to favor trajectories with better tool use and more accurate final solutions, while staying close to a reference policy for stability. 
Specifically, Eqn.~\eqref{eqn:agentic_dpo_loss} increases the relative log-likelihood of preferred trajectories $\tau^{+}$ over dispreferred ones $\tau^{-}$ compared to the reference policy. 
Unlike RLHF methods that require explicit reward modeling or value estimation, this formulation learns directly from pairwise trajectory preferences, making it suitable for multi-turn reasoning where step-wise rewards are difficult to define.

\noindent
\textbf{Limitations of Trajectory-level DPO.}
Although Trajectory-level DPO
enables stable optimization without explicit reward modeling, it still provides only coarse supervision for multi-turn reasoning. 
Since the preference signal is defined by the final outcome, all intermediate actions $\{a^t\}_{t=1}^T$ receive the same supervision regardless of their individual contributions. 
This makes it difficult to identify which reasoning steps cause success or failure, leading to inefficient policy improvement and high-variance optimization when trajectories share similar final outcomes but differ in intermediate quality. 

\subsubsection{Multi-level Preference Learning Approach to Agentic Learning}

To address the long-horizon problem aforementioned, we propose a new Multi-level Preference Learning (MPO) that decomposes trajectory-level preferences into multiple aligned state-level comparisons. 
Let a reasoning trajectory be denoted as 
$\tau=\{(s^t,a^t)\}_{t=1}^{T}$,
where each state
$s^t=(q,K,\mathcal{M}^t)$
consists of the query $q$, toolset $K$, and evolving memory $\mathcal{M}^t$ that summarizes the accumulated reasoning context up to step $t$. 
Since the preferred $\tau^+$ and dispreferred trajectories and $\tau^-$ may have different lengths and reasoning structures, we define an alignment set 
$
A \subseteq \{1,..,T^+\}\times\{1,..,T^-\}$ ($T^+$ and $T^-$ are the lenghts of $\tau^+$ and $\tau^-$) that pairs intermediate reasoning states according to their semantic similarity, enabling comparison between corresponding reasoning steps across trajectories.
For each aligned pair $(i,j)\in A$, we define the preference probability as follows:
\begin{equation}
\footnotesize
p_{ij} = \sigma\Big( \beta \big[ \log \pi_\theta(a_i^+|s_i^+) - \log \pi_\theta(a_j^-|s_j^-) \big] - \beta \big[ \log \pi_{\mathrm{ref}}(a_i^+|s_i^+) - \log \pi_{\mathrm{ref}}(a_j^-|s_j^-) \big] \Big).
\end{equation}
Then, our proposed \textbf{Multi-level Preference Optimization} objective can be defined as follows:
\begin{equation}
\label{eqn:flow_dpo}
\footnotesize
\mathcal{L}_{\mathrm{MPO}}(\pi_\theta,\pi_{\mathrm{ref}}) = - \mathbb{E}_{(\tau^+,\tau^-)} \frac{1}{|A|} \sum_{(i,j)\in A} \log p_{ij}.
\end{equation}
Unlike Trajectory-level DPO, 
our MPO objective introduces multiple intermediate preference constraints across aligned reasoning states. 
This formulation enables more effective credit assignment across multi-turn reasoning steps, allowing the policy to identify which intermediate decisions contribute positively or negatively to the final outcome. 
Consequently, MPO mitigates the long-horizon dependency in agentic learning and provides a more stable optimization signal across trajectories with varying reasoning structures and lengths.

\noindent
\textbf{Rollout Scoring for On-the-Fly Preference Construction.}
For each query $q$, we generate a small set of $K$ rollouts $\{\tau_k\}_{k=1}^{K}$ from the current policy $\pi_{\theta}$ and score each trajectory by both final correctness and reasoning efficiency. 
Let $R(\tau_k)$ denote the final reward, $U(\tau_k)$ the number of tool calls, and $T_k$ the trajectory length. 
We define $S(\tau_k) = R(\tau_k) - 0.05\frac{U(\tau_k)}{U_{\max}} - 0.05\frac{T_k}{T_{\max}}$, where $U_{\max}$ and $T_{\max}$ are computed within the current rollout group. 
The first term with a higher coefficient prioritizes trajectories that produce correct final solutions, while the second and third terms favor trajectories that solve the query using fewer tool invocations and fewer reasoning steps.
We then select the highest-scored rollout as $\tau^{+}$ and the lowest-scored rollout as $\tau^{-}$, enabling on-the-fly preference construction that encourages correct and efficient agentic reasoning.

\noindent
\textbf{Alignment Set $A$ Construction.}
Preferred and dispreferred trajectories may have different lengths and reasoning structures, so step-wise correspondence is not directly available. 
We therefore construct an alignment set
$
A \subseteq \{1,\dots,T^+\}\times\{1,\dots,T^-\}
$
that pairs semantically related states. 
Each state $s^t=(q,a^t,k^t,\mathcal{M}^t)$ includes the planner action, selected tool, and evolving memory.
For efficiency, we use monotonic dynamic programming alignment based on structured state similarity. 
For each preferred state $s_i^+$, we select:
\begin{equation}
\small
j^*(i)
=
\arg\max_{j \in \{1,\dots,T^-\}}
\mathrm{sim}(s_i^+,s_j^-),
\qquad
\text{s.t. } j^*(i+1)\ge j^*(i).
\end{equation}
The similarity is computed as $\mathrm{sim}(s_i^+,s_j^-) = \mathbf{1}(a_i^+=a_j^-) + \mathbf{1}(k_i^+=k_j^-) + \mathrm{sim}_m(\mathcal{M}_i^+,\mathcal{M}_j^-)$
where $\mathrm{sim}_m$ is the Jaccard overlap between tokenized memory states. 
The final alignment set is $A=\{(i,j^*(i))\}_{i=1}^{T^+}$, 
with monotonic filtering to retain at most $\min(T^+,T^-)$ aligned pairs. 
This construction preserves temporal consistency with low overhead, allowing MPO to propagate preference signals to intermediate reasoning states.

\subsubsection{Theoretical Analysis of Preference Learning in Agentic Learning}

Recent work shows that GRPO provides stable learning signals for multi-turn agentic reasoning through intra-group relative advantages \cite{li2025flow}.
Following prior interpretations of DPO and GRPO as contrastive objectives \cite{wu2025takes}, we show that K-GRPO can be decomposed into a positively weighted combination of pairwise preference gradients.
The, we prove that our preference learning induces an equivalent optimization direction as GRPO while avoiding explicit value estimation.
For simplicity, we present the equivalence at the trajectory level since our MPO preserves the same pairwise structure while extending supervision to intermediate reasoning states.
Let $\mathcal G(q) = \{\tau_1,\dots,\tau_K\}$ where $\tau_i \sim \pi_\theta(\cdot|q)$ be a group of $K \ge 2$ sampled trajectories.
Then, we define preferred $\mathcal P(q)$ and dispreferred $\mathcal N(q) $ trajectory sets, $m = |\mathcal P(q)|$, $n = |\mathcal N(q)|$, $m+n=K$, and assume non-degenerate groups ($1 \le m \le K-1$).
The trajectory-level score function can be defined as $s_\theta(q,\tau) = \nabla_\theta \log \pi_\theta(\tau|q)$.

\begin{theorem}[Pairwise Equivalence of DPO and K-GRPO up to Positive Weighting]
\label{thm:KGRPO_DPO_equivalence}

Under unclipped trajectory-level updates with binary rewards, the K-GRPO gradient can be written as
\begin{equation}
\footnotesize
g_\theta^{\mathrm{K\text{-}GRPO}}(q)
=
\alpha(q)
\;
\mathbb E_{\tau^+ \sim \mathrm{Unif}(\mathcal P(q)),
\;
\tau^- \sim \mathrm{Unif}(\mathcal N(q))}
\Big[
s_\theta(q,\tau^+) - s_\theta(q,\tau^-)
\Big].
\end{equation}
where the group-dependent positive weight is $\alpha(q) = \sqrt{\hat p_q(1-\hat p_q)} = \frac{\sqrt{mn}}{K}$, $\hat p_q = \frac{m}{K}$.
Moreover, the Trajectory-level DPO gradient can be written as $\nabla_\theta L_{\mathrm{DPO}} = -\mathbb E_{q,\tau^+,\tau^-} \Big[ w_\theta(q,\tau^+,\tau^-) \big( s_\theta(q,\tau^+) - s_\theta(q,\tau^-) \big) \Big]$, s
where $w_\theta(q,\tau^+,\tau^-) = \beta \sigma(-\Delta_\theta(q,\tau^+,\tau^-)) >0$ with $\Delta_\theta(q,\tau^+,\tau^-)$ is the DPO logit margin.
Therefore, K-GRPO and DPO share the same pairwise gradient basis $s_\theta(q,\tau^+) - s_\theta(q,\tau^-)$
and differ only by positive weighting (the proof is provided in the appendix).

\end{theorem}

\noindent
\textbf{Interpretation.}
Theorem~\ref{thm:KGRPO_DPO_equivalence} shows that GRPO implicitly performs pairwise preference learning over preferred-dispreferred trajectory pairs, while DPO makes this structure explicit through adaptive logistic weighting and reference-policy regularization.
Thus, DPO preserves the same optimization direction as GRPO but avoids explicit value estimation and repeated on-policy sampling, yielding a more efficient framework for long-horizon agentic reasoning.

\subsection{Imbalanced Learning in Agentic Reasoning}
\label{subsec:focal_flow_dpo}

In multi-turn agentic reasoning, data imbalance propagates through rollout generation and further induces imbalance in the aligned state pairs $A$. 
Since $\pi_\theta$ frequently generates common reasoning patterns, e.g., repeated tool selections or standard intermediate steps, sampled rollouts contain many similar trajectory segments. 
As a result, alignment construction produces many \emph{easy} pairs where the preferred action already has high probability, while rare but critical transitions, e.g., correcting an incorrect tool call or revising evidence, yield only a few \emph{hard} pairs requiring meaningful policy updates. 
Thus, $A$ becomes dominated by easy-majority pairs, biasing gradients toward frequent reasoning structures and under-optimizing rare but important decisions.
Formally, under Eqn.~\eqref{eqn:flow_dpo}, the gradient magnitude of each aligned pair is proportional to $\psi_0(p) = \left| \frac{\partial (-\log p_{i,j})}{\partial z} \right| = 1-p_{i,j}$,  where $z$ is the logit margin. 
Although easy pairs with $p_{i,j}\approx 1$ have small individual gradients, their large quantity can still dominate the aggregate update. 
This leads to biased optimization and slow improvement on underrepresented hard decisions.

\noindent
\textbf{Fair Multi-level Preference Optimization.}
To mitigate the bias caused imbalance data, 
inspired by \cite{lin2017focal, truong2026phi}, we introduce a focal weighting mechanism that down-weights easy preference pairs and emphasizes hard pairs.
Formally, we formulate our \textbf{$\Phi$-MPO} objective as follows:
\begin{equation}
\label{eq:focal_flow_dpo}
\small
\mathcal L_{\mathrm{FairMPO}} = - \mathbb{E}_{(\tau^+,\tau^-)}  \frac{1}{|A|} \sum_{(i,j)\in A} (1-p_{i,j})^\gamma \log p_{i,j}, 
\end{equation}
where $\gamma \ge 0$ is the focal parameter.
Compared to Eqn.~\eqref{eqn:flow_dpo}, the factor $(1-p_{i,j})^\gamma$ suppresses the contribution of easy pairs ($p_{i,j}\approx 1$) and increases the importance of hard pairs with smaller $p_{i,j}$.

\begin{theorem}[Suppression of Easy-Majority Dominance in $\Phi$-MPO]\label{thm:focal-flow-dpo}
Let
us define $\mathcal{L}_\gamma(p)=-(1-p)^\gamma \log p$, $p=\sigma(z)$, $\gamma\ge 0$. Then, the margin-gradient magnitude can be formed as $\psi_\gamma(p):=\left|\frac{\partial \mathcal{L}_\gamma(p)}{\partial z}\right|$, 
Suppose that, at level $\ell$, the aligned pairs are partitioned into an easy-majority set $E^{(\ell)}$ and a hard-minority set $H^{(\ell)}$ such that $p_{ij}^{(\ell)} \ge 1-\varepsilon \quad \forall (i,j)\in E^{(\ell)}$ and $p_{ij}^{(\ell)} \le 1-\delta \quad \forall (i,j)\in H^{(\ell)}$
for some $0<\varepsilon<\delta<1$.
Then
\begin{equation}
\footnotesize
\begin{split}
\sum_{(i,j)\in E^{(\ell)}} \psi_\gamma\!\left(p_{ij}^{(\ell)}\right)
\le
|E^{(\ell)}|(\gamma+1)\varepsilon^{\gamma+1}
 \quad &\text{and} \quad
\sum_{(i,j)\in H^{(\ell)}} \psi_\gamma\!\left(p_{ij}^{(\ell)}\right)
\ge
|H^{(\ell)}|\delta^{\gamma+1}
\\
\Rightarrow
\frac{
\sum_{(i,j)\in E^{(\ell)}} \psi_\gamma\!\left(p_{ij}^{(\ell)}\right)
}{
\sum_{(i,j)\in H^{(\ell)}} \psi_\gamma\!\left(p_{ij}^{(\ell)}\right)
}
&\le
\frac{|E^{(\ell)}|}{|H^{(\ell)}|}
(\gamma+1)\left(\frac{\varepsilon}{\delta}\right)^{\gamma+1}.
\end{split}
\end{equation}
Therefore, $\Phi$-MPO suppresses the optimization influence of overrepresented easy pairs and mitigates imbalance caused by easy-majority dominance (the proof is provided in the appendix).
\end{theorem}

\paragraph{Interpretation.}
When $\gamma=0$, $\Phi$-MPO recovers standard MPO, whose easy-majority versus hard-minority gradient contribution scales with the frequency ratio $|E^{(\ell)}|/|H^{(\ell)}|$.
As $\gamma$ increases, easy-pair contributions decay at the higher-order rate $\varepsilon^{\gamma+1}$, while hard pairs retain meaningful gradients.
Thus, $\Phi$-MPO prevents frequent easy decisions from dominating optimization, yielding fairer updates and better adaptation to rare but informative reasoning transitions.

\begin{figure}[!b]
    \vspace{-4mm}
    \centering
    \includegraphics[width=1.0\linewidth]{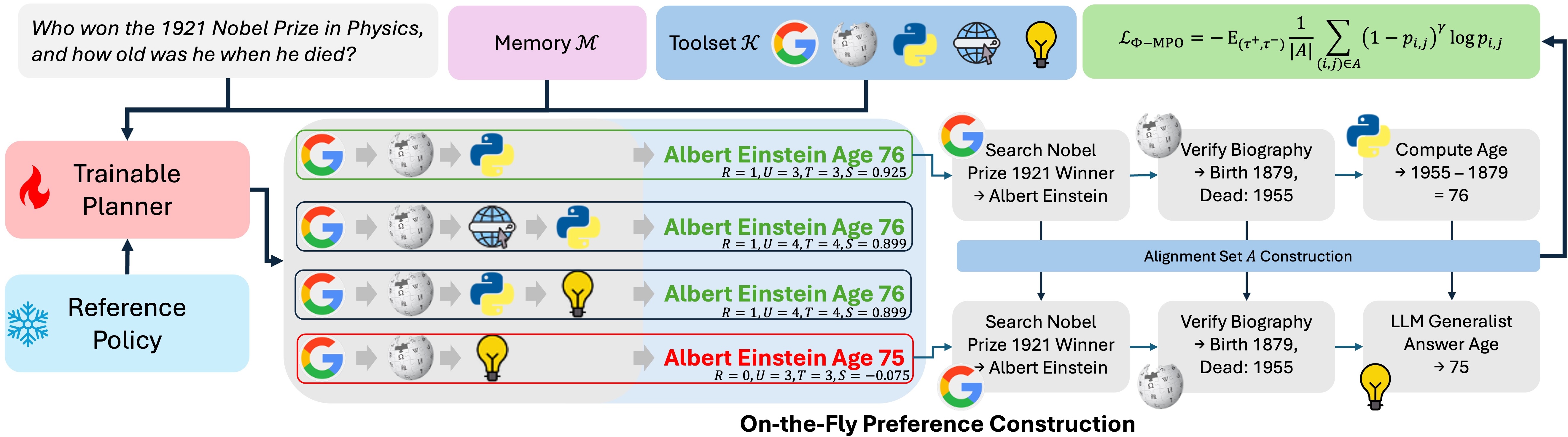}
    \vspace{-5mm}
    \caption{\textbf{Our Proposed $\Phi$-MPO Framework}.}
    \label{fig:framework}
\end{figure}

\textbf{Final Agentic Learning Objective.}
In our framework, we combine trajectory-level and MPO to improve both global task success and intermediate reasoning quality. 
Trajectory-level DPO encourages trajectories with correct final outcomes, while MPO provides denser supervision over aligned reasoning states, improving tool use and evidence accumulation:
\begin{equation}
\footnotesize
\mathcal L_{\mathrm{Final}}
=
\lambda_{\mathrm{DPO}}\mathcal{L}_{\mathrm{DPO}}
+
\lambda_{\mathrm{\Phi-MPO}}\mathcal{L}_{\mathrm{\Phi-MPO}},
\end{equation}
where $\lambda_{\mathrm{DPO}}$ and $\lambda_{\mathrm{\Phi-MPO}}$ balance the two losses. 
We also apply the focal weighting in Eqn.~\eqref{eq:focal_flow_dpo} to $\mathcal{L}_{\mathrm{DPO}}$ to reduce easy-majority dominance. 
Figure~\ref{fig:framework} illustrates the overall framework.

\section{Experiments}

\subsection{Implementation and Benchmarks}

\textbf{Implementation.} 
We adopt the implementation of \cite{li2025flow} and instantiate the \textit{Action Planner}, \textit{Tool Executor}, \textit{Execution Verifier}, and \textit{Solution Generator} with Qwen2.5-7B-Instruct \cite{qwen2.5}.
The agent uses five tools: \textit{Base Generator}, \textit{Python Coder}, \textit{Google Search}, \textit{Wikipedia Search}, and \textit{Web Search}. 
Following standard protocols \cite{li2025flow, jin2025search}, we train with learning rate $10^{-6}$, penalty $\beta=0.001$, planner temperature $0.5$, maximum planner length $2048$, batch size $12$, and $K=4$ rollouts per input. 
We set the maximum number of turns to $3$ during training and $10$ during evaluation, use GPT-4o as the LLM judge for the reward function, and run all tool LLMs with temperature $0.0$ for deterministic execution. 
Training is conducted on $12$ NVIDIA L40S GPUs. 
We set $\gamma=2.0$ and $\lambda_{\mathrm{DPO}}=\lambda_{\mathrm{\Phi-MPO}}=1.0$. 
Detailed tool definitions and memory updates are provided in the Appendix.

\textbf{Datasets and Benchmarks.}
Following prior protocols \cite{jin2025search, li2025flow}, our framework is trained on a mixture of Search-R1 \cite{jin2025search} and DeepMath \cite{he2025deepmath}.
For fair comparisons, we conduct our evaluation on four benchmarks:
(1) \textit{\textbf{Knowledge-Intensive Search}} includes Bamboogle \cite{press2023measuring} , 2Wiki \cite{ho2020constructing}, HotpotQA \cite{yang2018hotpotqa}, and Musique \cite{trivedi2022musique};
(2) \textit{\textbf{Agentic Reasoning}} consists of GAIA \cite{mialon2023gaia};
(3) \textit{\textbf{Mathematical Reasoning}} includes AIME2024 \cite{aops_aime}, AMC23 \cite{AMC23}, and GameOf24 \cite{lightman2023let}.
and (4) \textit{\textbf{Scientific Reasoning}} includes GPQA \cite{rein2024gpqa} and MedQA \cite{yang2025llm}.

\subsection{Main Results}

\begin{table}[t]
    \centering
    \begin{minipage}{0.49\linewidth}
    \centering
    \setlength{\tabcolsep}{1pt}
    \caption{\textbf{Accuracy Results on Search \& Agentic Reasoning Benchmarks.}}
    \label{tab:search-agentic}
    \resizebox{1.0\linewidth}{!}{
    \begin{tabular}{ll|cccc|c}
    \hline
    \multirow{2}{*}{Method} & \multirow{2}{*}{Model Size} & \multicolumn{4}{c|}{Knowledge-Intensive Search}   & Agentic \\
                           &                       & Bamboogle & 2Wiki & HotpotQA & Musique & GAIA    \\
    \hline
    Qwen2.5-7B-Inst    &   7B             & 12.0     & 23.0 & 21.0    & 6.0    & 3.2    \\
    Qwen2.5-14B-Inst  &   14B            & 21.6     & 26.7 & 20.0    & 8.0    & 5.5    \\
    Qwen2.5-32B-Inst  &   32B            & 24.0     & 26.7 & 27.0    & 6.0    & 9.5    \\
    Llama-3.3-70B-Inst &   70B            & 18.4     & 22.7 & 52.0    & 16.0   & 3.2    \\
    \hline
    GPT-4o-mini \cite{hurst2024gpt}            & 8B                    & 40.8     & 35.6 & 41.0    & 15.0   & 7.1    \\
    GPT-4o \cite{hurst2024gpt}                & 200B                  & 68.8     & 49.5 & 54.0    & 24.0   & 17.3   \\
    \hline
    SFT & 7B-Inst               & 12.0     & 25.9 & 22.0    & 6.6    & 3.2    \\
    Iter-RetGen \cite{shao2023enhancing}            & 7B-Inst               & 36.8     & 33.6 & 37.4    & 17.8   & 3.9    \\
    Search-R1 \cite{jin2025search}             & 7B-Inst               & 43.2     & 38.2 & 37.0    & 14.6   & 19.1   \\
    ZeroSearch \cite{sun2025zerosearch}            & 7B-Base               & 27.8     & 35.2 & 34.6    & 18.0   & 16.5   \\
    ReSearch \cite{chen2025learning}               & 7B-Base               & 42.4     & 47.6 & 43.5    & 22.3   & 17.3   \\
    StepSearch \cite{wang2025stepsearch}             & 7B-Base               & 40.0     & 36.6 & 38.6    & 22.6   & -       \\
    VerlTool \cite{jiang2025verltool}              & 7B-Base               & 46.4     & 45.3 & 44.8    & 19.3   & 11.2   \\
    AutoGen \cite{wu2024autogen}               & 7B-Inst               & 59.6     & 44.0 & 50.0    & 15.9   & 6.3    \\
    AgentFlow \cite{li2025flow}             & 7B-Inst               & 58.4     & 60.0 & 51.3    & 19.2   & 17.2   \\
    FlowGRPO \cite{li2025flow}   & 7B-Inst               & 69.6     & 77.2 & 57.0    & 25.3   & 33.1   \\
    \hline
    $\Phi$-MPO                   & 7B-Inst               &  \textbf{81.6} & \textbf{81.5} & \textbf{69.0} & \textbf{35.0} & \textbf{41.7} \\
    $\Phi$-MPO                  &   Qwen3.5-9B  &    \textbf{85.6}	&  \textbf{85.0}	&  \textbf{77.0}	&  \textbf{40.0}	& \textbf{49.6} \\
    
    \hline
    \end{tabular}
    }
    \end{minipage}
    \hfill
    \begin{minipage}{0.49\linewidth}
    \centering
    \setlength{\tabcolsep}{1pt}
    \caption{\textbf{Accuracy Results on Math \& Scientific Reasoning Benchmarks.} (7B-Inst: Qwen2.5-7B-Instruct, 7B-Base: Qwen2.5-7B-Base)}
    \label{tab:math-science}
    \resizebox{1.0\linewidth}{!}{
    \begin{tabular}{ll|ccc|cc}
    \hline
    \multirow{2}{*}{Method}  & \multirow{2}{*}{Model Size} & \multicolumn{3}{c|}{Mathematical Reasoning} & \multicolumn{2}{c}{Scientific Reasoning} \\
                            &                       & AIME24       & AMC23       & GameOf24      & GPQA               & MedQA               \\
    \hline
    Qwen2.5-7B-Inst    & 7B               & 6.7          & 47.5        & 33.0          & 34.0               & 66.0                \\
    Qwen2.5-14B-Inst   & 14B              & 6.7          & 60.0        & 25.0          & 31.0               & 75.0                \\
    Llama-3.3-70B-Inst  & 70B              & 6.7          & 47.5        & 31.0          & 35.0               & 67.0                \\
    Llama-3.1-405B-Inst & 405B             & 26.7         & 47.5        & 23.0          & 30.0               & 62.0                \\
    \hline
    GPT-4o-mini \cite{hurst2024gpt}             & 8B              & 13.3         & 57.5        & 16.0          & 27.0               & 66.0                \\
    GPT-4o \cite{hurst2024gpt}                  & 200B            & 13.3         & 60.0        & 32.0          & 31.0               & 60.0                \\
    \hline
    SFT  & 7B-Inst               & 6.7          & 47.5        & 33.0          & 34.0               & 66.0                \\
    SimpleRL \cite{zeng2025simplerl}         & 7B-Base               & 16.7         & 60.0        & 33.0          & 45.0               & 65.0                \\
    Open-Reasoner \cite{hu2025open}      & 7B-Base               & 16.7         & 54.9        & 32.0          & 34.0               & 54.0                \\
    General-Reasoner \cite{ma2025general}       & 7B-Base               & 13.3         & 55.0        & 33.0          & 35.5               & 61.0                \\
    Luffy \cite{yan2025learning}                   & 7B-Inst               & 30.7         & 44.8        & 33.0          & 34.0               & 77.0                \\
    TIR \cite{yang2024qwen2}                    & 7B-Inst               & 10.0         & 50.0        & 33.0          & 42.0               & 76.8                \\
    ToRL  \cite{li2025torl}                  & 7B-Inst               & 20.0         & 60.0        & 31.0          & 35.0               & 76.5                \\
    AutoGen \cite{wu2024autogen}                 & 7B-Inst               & 13.3         & 57.5        & 24.0          & 42.0               & 72.0                \\
    AgentFlow \cite{li2025flow}              & 7B-Inst               & 16.7         & 47.4        & 31.0          & 37.0               & 76.0                \\
    Flow-GRPO \cite{li2025flow}   & 7B-Inst               & 40.0         & 61.5        & 53.0          & 47.0               & 80.0           \\
    \hline
    $\Phi$-MPO                   & 7B-Inst               & \textbf{50.0} & \textbf{72.5} & \textbf{65.0} & \textbf{59.0} & \textbf{86.0} \\
    $\Phi$-MPO                  &   Qwen3.5-9B  &   \textbf{73.3}	& \textbf{80.0}	& \textbf{75.0}	& \textbf{64.0}	& \textbf{90.0} \\
    \hline
    \end{tabular}
    }
    \end{minipage}
    \vspace{-4mm}
\end{table}

\textbf{Search-Intensive and Agentic Benchmarks.}
Table~\ref{tab:search-agentic} shows that our method consistently achieves the best performance across all knowledge-intensive search benchmarks and the agentic reasoning benchmark GAIA. 
With the Qwen2.5-7B-Instruct backbone, our method attains 81.6\% on Bamboogle, 81.5\% on 2Wiki, 69.0\% on HotpotQA, 35.0\% on Musique, and 41.7\% on GAIA, consistently outperforming prior search-based, agentic, and reinforcement learning baselines, including FlowGRPO \cite{li2025flow}. 
Notably, our 7B model even surpasses GPT-4o across all five benchmarks, highlighting the effectiveness of our proposed framework for both knowledge-intensive retrieval and multi-turn agentic reasoning. 
When using the stronger Qwen3.5-9B \cite{qwen3.5} backbone, the performance is further improved to 85.6\%, 85.0\%, 77.0\%, 40.0\%, and 49.6\%, respectively, which further confirms the robustness and scalability of our approach.
Figure \ref{fig:result} shows a case study of our approach.

\begin{wrapfigure}[15]{r}{0.5\linewidth}
    \centering
    \vspace{-6mm}
    \includegraphics[width=1.0\linewidth]{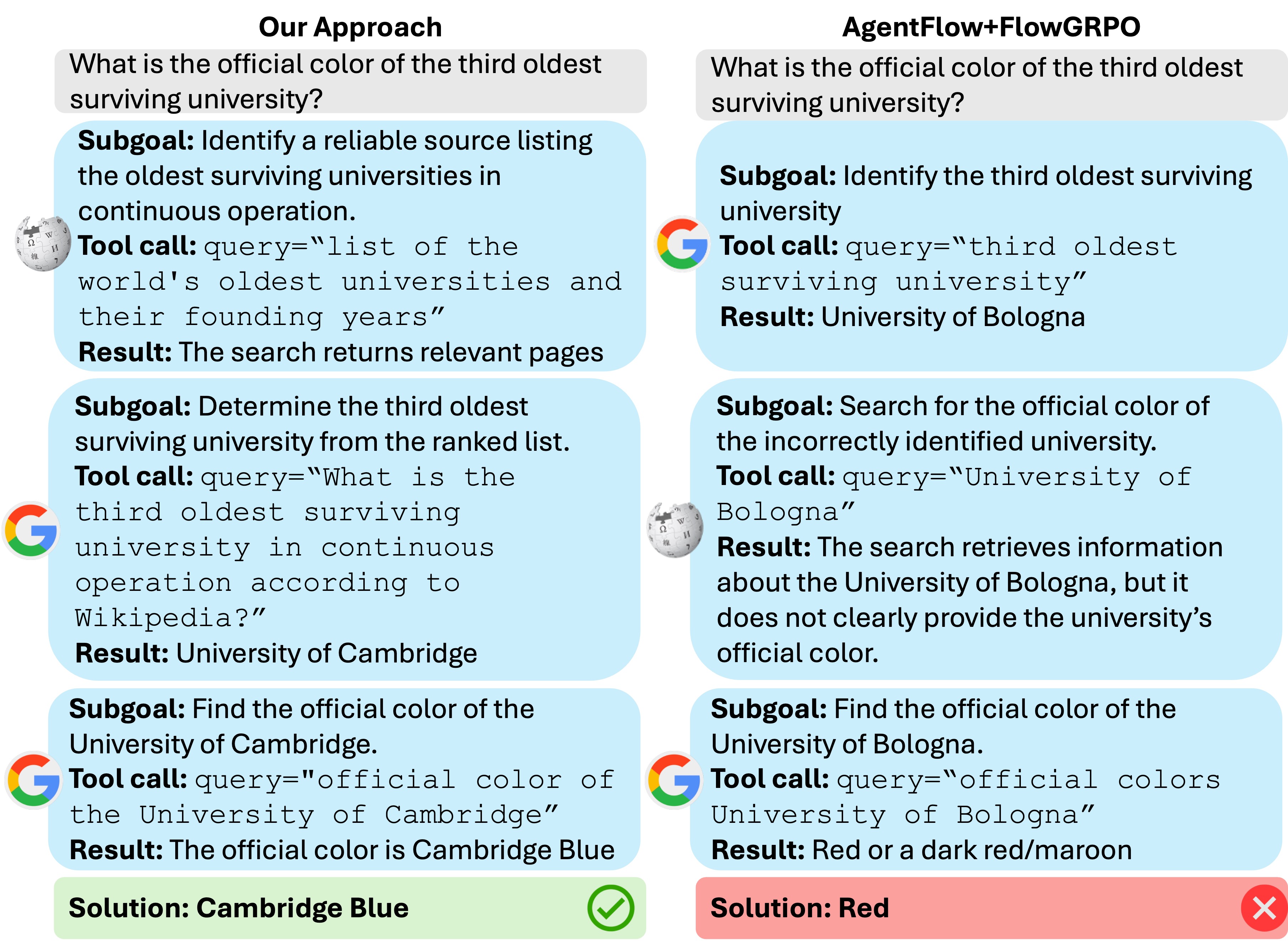}
    \vspace{-5mm}
    \caption{Our Case Study Example.
    }
    \label{fig:result}
\end{wrapfigure}
\textbf{Mathematical and Scientific Reasoning Benchmarks.}
The results in Table~\ref{tab:math-science} further demonstrate the strong effectiveness of our approach on both mathematical and scientific reasoning tasks.
Using the Qwen2.5-7B-Instruct backbone, $\Phi$-MPO achieves 50.0\% on AIME24, 72.5\% on AMC23, 65.0\% on GameOf24, 59.0\% on GPQA, and 86.0\% on MedQA, establishing the best performance across all benchmarks.
Compared prior Flow-GRPO \cite{li2025flow}, our approach obtains consistent gains on every task, with especially notable improvements on AIME24, GameOf24, and GPQA.
These results suggest that our approach is particularly effective for challenging multi-step reasoning problems that require both precise planning and reliable tool use.
When scaling to Qwen3.5-9B \cite{qwen3.5}, the performance is further improved,
which again confirms our robustness and scalability.

\subsection{Ablation Studies}

\textbf{Effectiveness of Different Training Approaches.} 
Table~\ref{tab:abl-multi-level-loss} compares Frozen, (Supervised Finetuning) SFT, Flow-GRPO, and our method. 
The Frozen baseline performs competitively on Bamboogle (58.4\%) and 2Wiki (60.0\%), but is limited on GAIA (17.2\%) and AIME24 (16.7\%). 
SFT performs worse across all benchmarks, suggesting that standard supervised tuning is insufficient for multi-turn agentic reasoning. 
Flow-GRPO improves performance to 69.6\%, 77.2\%, 33.1\%, and 40.0\%, respectively, confirming the benefit of reinforcement learning for long-horizon optimization. 
Our method further outperforms Flow-GRPO on all benchmarks, demonstrating the effectiveness of preference optimization for agentic learning.

\begin{wraptable}[7]{r}{0.5\linewidth}
\centering
\vspace{-4mm}
\caption{Effectiveness of Our Proposed Objective.}
\label{tab:abl-multi-level-loss}
\setlength{\tabcolsep}{2pt}
\vspace{-2mm}
\resizebox{1.0\linewidth}{!}{
\begin{tabular}{ccc|cccc}
\hline
DPO         & Multi-Level        & Fair        & Bamboogle & 2Wiki & GAIA & AIME24 \\
\hline
\cmark           &            &              & 61.6\% & 61.5\% & 21.3\% & 30.0\% \\ 
\cmark           & \cmark     &              &  76.0\% & 76.0\% & 36.2\% & 43.3\% \\
\cmark           & \cmark     & \cmark       &      \textbf{81.6\%} & \textbf{81.5\%} & \textbf{41.7\%} & \textbf{50.0\%} \\

\hline
\multicolumn{3}{l|}{Frozen}                   & 58.4\%      & 60.0\%  & 17.2\% & 16.7\%   \\
\multicolumn{3}{l|}{Supervised   Fine-Tuning} & 30.4\%      & 32.7\%  & 6.3\%  & 3.3\%    \\
\multicolumn{3}{l|}{Flow-GRPO}                & 69.6\%      & 77.2\%  & 33.1\% & 40.0\%  \\
\hline
\end{tabular}
}
\end{wraptable}

\textbf{Effectiveness of Fair MPO.}
As shown in Table~\ref{tab:abl-multi-level-loss}, Trajactory-level DPO achieves 61.6\% on Bamboogle, 61.5\% on 2Wiki, 21.3\% on GAIA, and 30.0\% on AIME24, outperforming Frozen and SFT baselines on most tasks. 
Adding MPO improves performance to 76.0\%, 76.0\%, 36.2\%, and 43.3\%, showing the benefit of intermediate flow supervision. 
Our final $\Phi$-MPO achieves the best results, reaching 81.6\%, 81.5\%, 41.7\%, and 50.0\%, respectively. 
These results confirm that multi-level supervision and focal balancing both improve agentic preference learning under imbalanced data.

\begin{wraptable}[6]{r}{0.45\linewidth}
\centering
\vspace{-4mm}
\caption{Effectiveness of Number of Turns.}
\label{tab:abl-turns}
\vspace{-2mm}
\resizebox{1.0\linewidth}{!}{
\begin{tabular}{c|cccc}
\hline
$T$ & Bamboogle & 2Wiki & GAIA & AIME24 \\
\hline
3 & 60.8\% & 61.0\% & 21.3\% & 26.7\% \\
5 & 69.6\% & 70.0\% & 29.9\% & 36.7\% \\
7 & 76.8\% & 77.0\% & 37.0\% & 43.3\% \\
10   &      \textbf{81.6\%} & \textbf{81.5\%} & \textbf{41.7\%} & \textbf{50.0\%} \\

\hline
\end{tabular}
}
\end{wraptable}

\textbf{Effectiveness of Maximum Number of Turns.}
As shown in Table~\ref{tab:abl-turns}, increasing the turn budget consistently improves performance across all benchmarks. 
With $T=3$, the model achieves only 60.8\% on Bamboogle, 61.0\% on 2Wiki, 21.3\% on GAIA, and 26.7\% on AIME24. 
Performance steadily improves as $T$ increases to $5$ and $7$, with the best results at $T=10$: 81.6\%, 81.5\%, 41.7\%, and 50.0\%, respectively. 
These results show that more interaction steps enable deeper planning, tool use, and recovery from early mistakes in long-horizon reasoning.

\begin{wraptable}[7]{r}{0.45\linewidth}
\centering
\vspace{-4mm}
\caption{Effectiveness of Hyper-Parameter $\gamma$.}
\label{tab:abl-gamma}
\vspace{-2mm}
\resizebox{1.0\linewidth}{!}{
\begin{tabular}{c|cccc}
\hline
$\gamma$ & Bamboogle & 2Wiki & GAIA & AIME24 \\
\hline
0.0 & 76.0\%          & 76.0\%          & 36.2\%          & 43.3\%          \\
0.5 & 76.0\%          & 76.5\%          & 37.0\%          & 46.7\%          \\
1.0 & 78.4\%          & 78.0\%          & 37.8\%          & 46.7\%          \\
2.0 & \textbf{81.6\%} & \textbf{81.5\%} & \textbf{41.7\%} & \textbf{50.0\%} \\
5.0 & 75.2\%          & 75.5\%          & 35.4\%          & 43.3\%         \\    
\hline
\end{tabular}
}
\end{wraptable}

\textbf{Effectiveness of Focal Parameter $\gamma$.}
Table~\ref{tab:abl-gamma} studies the focal hyper-parameter $\gamma$, which controls the emphasis on hard reasoning decisions. 
When $\gamma=0.0$, the objective reduces to standard MPO, achieving 76.0\% on Bamboogle, 76.0\% on 2Wiki, 36.2\% on GAIA, and 43.3\% on AIME24. 
Increasing $\gamma$ from $0.5$ to $2.0$ consistently improves performance, with $\gamma=2.0$ achieving the best results: 81.6\%, 81.5\%, 41.7\%, and 50.0\%, respectively. 
However, $\gamma=5.0$ degrades performance, suggesting that overly strong focal weighting suppresses too many gradients and weakens learning. 
These results show that a moderate focal factor best balances easy-majority patterns and informative hard decisions.

\begin{wraptable}[8]{r}{0.5\linewidth}
\centering
\vspace{-4mm}
\caption{Effectiveness of Model Size.}
\label{tab:abl-model-size}
\setlength{\tabcolsep}{2pt}
\vspace{-2mm}
\resizebox{1.0\linewidth}{!}{
\begin{tabular}{ll|cccc}
\hline
Method                     & Size & Bamboogle & 2Wiki & GAIA & AIME24 \\
\hline
\multirow{2}{*}{Frozen}    & Qwen2.5-3B-Instruct   & 53.6\%      & 63.0\%  & 14.3\% & 13.3\%   \\
                           & Qwen2.5-7B-Instruct   & 58.4\%      & 60.0\%  & 17.2\% & 16.7\%   \\
\hline
\multirow{2}{*}{AgentFlow} & Qwen2.5-3B-Instruct   & 68.8\%      & 72.3\%  & 29.1\% & 20.0\%   \\
                           & Qwen2.5-7B-Instruct   & 69.6\%      & 77.2\%  & 33.1\% & 40.0\%   \\
\hline
\multirow{2}{*}{Ours}      & Qwen2.5-3B-Instruct   &     72.0\% & 72.5\% & 32.3\% & 40.0\% \\
                           & Qwen2.5-7B-Instruct   &     \textbf{81.6\%} & \textbf{81.5\%} & \textbf{41.7\%} & \textbf{50.0\%} \\
\hline
\multirow{2}{*}{Ours}      & Qwen3.5-4B   & 74.4\%          & 74.0\%          & 34.6\%          & 33.3\%          \\
                           & Qwen3.5-9B   & \textbf{85.6\%} & \textbf{85.0\%} & \textbf{49.6\%} & \textbf{73.3\%} \\

\hline
\end{tabular}
}
\end{wraptable}

\textbf{Effectiveness of Different Backbone.}
As in Table~\ref{tab:abl-model-size} , across Qwen2.5-3B-Instruct and Qwen2.5-7B-Instruct, our method consistently outperforms the Frozen and AgentFlow with FlowGRPO baselines on all benchmarks, showing robustness across model capacities. 
Notably, our Qwen2.5-3B variant already surpasses AgentFlow with Qwen2.5-7B on several tasks, achieving 72.0\% on Bamboogle, 72.5\% on 2Wiki, 32.3\% on GAIA, and 40.0\% on AIME24. 
Scaling to Qwen2.5-7B further improves performance to 81.6\%, 81.5\%, 41.7\%, and 50.0\%, respectively. 
The same trend holds for Qwen3.5, where Qwen3.5-9B outperforms Qwen3.5-4B and reaches 85.6\% on Bamboogle, 85.0\% on 2Wiki, 49.6\% on GAIA, and 73.3\% on AIME24. 
These results show that our method generalizes across backbone families and scales effectively with stronger models.

\vspace{-2mm}
\section{Conclusions and Limitations}
\vspace{-2mm}

\noindent
\textbf{Conclusions.}
This paper has presented a novel $\Phi$-MPO approach to agentic learning. In particular, our proposed learning objective has been introduced to address both long-horizon and data imbalance problems in multi-turn agentic reasoning. Our theoretical analysis has also demonstrated the effectiveness and robustness of the proposed framework. Extensive experiments on multiple benchmarks have further confirmed the effectiveness of our method compared to other methods.

\noindent
\textbf{Limitations.}
Our paper adopts a set of design choices and hyper-parameters aligned with our theoretical analysis, but this also introduces several limitations, particularly in tuning the focal parameter $\gamma$, the weighting coefficients between trajectory-level and multi-level objectives, and the maximum rollout budget.
In addition, the effectiveness of our approach relies on the quality of the on-the-fly rollouts used to construct preference pairs, which may be sensitive to rollout diversity, reward reliability, and the stability of trajectory-level and step-level supervision. In long-horizon agentic settings, noisy tool feedback, imperfect final-outcome judgments, or insufficient exploration may lead to suboptimal preference signals and biased optimization. These limitations highlight the need for future studies on more robust preference optimization strategies for agentic learning.

\noindent
\textbf{Acknowledgment.} This work is partly supported by DOW/ARO (No. W911NF261A114), NSF/IIS CAREER (No. 2442295), NSF/IIS (No. 2501021), NSF/BIO (No. 2524623), and NSF/OIA (No. 2445877). This research is also supported by the Arkansas High Performance Computing Center which is funded through multiple National Science Foundation grants and the Arkansas Economic Development Commission.

\bibliographystyle{abbrv}
\bibliography{references}

\newpage

\appendix

\section*{\Large {Technical Appendices} }

\section{Proof of Theorems}

\subsection{Proof of Theorem \ref{thm:KGRPO_DPO_equivalence}}

Let $q$ denote a prompt and let $\mathcal G(q) = \{\tau_1,\dots,\tau_K\}, \qquad \tau_i \sim \pi_\theta(\cdot|q)$ be a group of $K \ge 2$ sampled trajectories.
Assume binary trajectory rewards $r(\tau) \in \{0,1\}$.
Then, we define preferred $\mathcal P(q)$ and dispreferred $\mathcal N(q) $ trajectory sets, $m = |\mathcal P(q)|$, $n = |\mathcal N(q)|$, $m+n=K$, and assume non-degenerate groups ($1 \le m \le K-1$).
The trajectory-level score function can be defined as $s_\theta(q,\tau) = \nabla_\theta \log \pi_\theta(\tau|q)$.

\noindent
\textbf{Theorem \ref{thm:KGRPO_DPO_equivalence}: Pairwise Equivalence of DPO and K-GRPO up to Positive Weighting.}
Under unclipped trajectory-level updates with binary rewards, the K-GRPO gradient can be written as

\begin{equation}
g_\theta^{\mathrm{K\text{-}GRPO}}(q)
=
\alpha(q)
\;
\mathbb E_{\tau^+ \sim \mathrm{Unif}(\mathcal P(q)),
\;
\tau^- \sim \mathrm{Unif}(\mathcal N(q))}
\Big[
s_\theta(q,\tau^+) - s_\theta(q,\tau^-)
\Big],
\end{equation}
where the group-dependent positive weight is $\alpha(q) = \sqrt{\hat p_q(1-\hat p_q)} = \frac{\sqrt{mn}}{K}$, and $\hat p_q = \frac{m}{K}$.
Moreover, the DPO gradient can be written as

\begin{equation}
\nabla_\theta L_{\mathrm{DPO}}
=
-\mathbb E_{q,\tau^+,\tau^-}
\Big[
w_\theta(q,\tau^+,\tau^-)
\big(
s_\theta(q,\tau^+) - s_\theta(q,\tau^-)
\big)
\Big],
\end{equation}
where $w_\theta(q,\tau^+,\tau^-) = \beta \sigma(-\Delta_\theta(q,\tau^+,\tau^-)) >0$.
Therefore K-GRPO and DPO share the same pairwise gradient basis $s_\theta(q,\tau^+) - s_\theta(q,\tau^-)$
and differ only by positive weighting. Hence K-GRPO is pairwise equivalent to DPO up to positive weighting.

\paragraph{Proof.}

GRPO uses group-normalized advantages \cite{wu2025takes, shao2024deepseekmath}. 
Under binary rewards, all preferred trajectories share one advantage value $A^+$ and all dispreferred trajectories share another value $A^-$:

\begin{equation}
A^+ =
\sqrt{\frac{1-\hat p_q}{\hat p_q}},
\qquad
A^- =
-\sqrt{\frac{\hat p_q}{1-\hat p_q}},
\qquad
\hat p_q = \frac{m}{K}.
\end{equation}

Ignoring PPO clipping and working at the trajectory level, the GRPO score-function gradient for one group is

\begin{equation}
g_\theta^{\mathrm{K\text{-}GRPO}}
=
\frac{1}{K}
\sum_{i=1}^{K}
A_i
s_\theta(q,\tau_i).
\end{equation}

Substituting the binary-reward advantages yields

\begin{equation}
g_\theta^{\mathrm{K\text{-}GRPO}}
=
\frac{1}{K}
\left(
\sum_{\tau^+ \in \mathcal P(q)} A^+ s_\theta(q,\tau^+)
+
\sum_{\tau^- \in \mathcal N(q)} A^- s_\theta(q,\tau^-)
\right).
\end{equation}

Since all preferred trajectories share $A^+$ and all dispreferred trajectories share $A^-$,

\begin{equation}
g_\theta^{\mathrm{K\text{-}GRPO}}
=
\frac{1}{K}
\left(
m A^+ \bar s_\theta^+
+
n A^- \bar s_\theta^-
\right),
\end{equation}

where

\begin{equation}
\bar s_\theta^+
=
\frac{1}{m}
\sum_{\tau^+ \in \mathcal P(q)}
s_\theta(q,\tau^+),
\qquad
\bar s_\theta^-
=
\frac{1}{n}
\sum_{\tau^- \in \mathcal N(q)}
s_\theta(q,\tau^-).
\end{equation}

Substituting the binary advantage values gives

\begin{equation}
mA^+ = \sqrt{mn},
\qquad
nA^- = -\sqrt{mn}.
\end{equation}

Thus

\begin{equation}
g_\theta^{\mathrm{K\text{-}GRPO}}
=
\frac{\sqrt{mn}}{K}
\left(
\bar s_\theta^+ - \bar s_\theta^-
\right).
\end{equation}

Observe that the difference between group averages can be written as the expectation over all preferred-dispreferred trajectory pairs:

\begin{equation}
\bar s_\theta^+ - \bar s_\theta^-
=
\mathbb E_{\tau^+,\tau^-}
\Big[
s_\theta(q,\tau^+) - s_\theta(q,\tau^-)
\Big].
\end{equation}

Therefore

\begin{equation}
g_\theta^{\mathrm{K\text{-}GRPO}}
=
\frac{\sqrt{mn}}{K}
\;
\mathbb E_{\tau^+,\tau^-}
\Big[
s_\theta(q,\tau^+) - s_\theta(q,\tau^-)
\Big].
\end{equation}

Thus K-GRPO is a positive-weighted average of pairwise preference gradients.

From the DPO objective \cite{wu2025takes}, we can derive the gradient of the DPO loss as follows:

\begin{equation}
\nabla_\theta \mathcal{L}_{\mathrm{DPO}}
=
-\mathbb E
\Big[
w_\theta(q,\tau^+,\tau^-)
\big(
s_\theta(q,\tau^+) - s_\theta(q,\tau^-)
\big)
\Big],
\end{equation}

with $w_\theta(q,\tau^+,\tau^-)>0$.

Hence both gradients lie in the cone generated by pairwise preference directions

\begin{equation}
s_\theta(q,\tau^+) - s_\theta(q,\tau^-),
\end{equation}

differing only by positive weighting.

\hfill $\square$

\subsection{Proof of Theorem \ref{thm:focal-flow-dpo}}

\noindent
\textbf{Theorem \ref{thm:focal-flow-dpo}: Suppression of Easy-Majority Dominance in $\Phi$-MPO.} Let
us define $\mathcal{L}_\gamma(p)=-(1-p)^\gamma \log p$, $p=\sigma(z)$, $\gamma\ge 0$. Then, the margin-gradient magnitude can be formed as $\psi_\gamma(p):=\left|\frac{\partial \mathcal{L}_\gamma(p)}{\partial z}\right|$, 
Suppose that, at level $\ell$, the aligned pairs are partitioned into an easy-majority set $E^{(\ell)}$ and a hard-minority set $H^{(\ell)}$ such that
\[
p_{ij}^{(\ell)} \ge 1-\varepsilon \quad \forall (i,j)\in E^{(\ell)},
\qquad
p_{ij}^{(\ell)} \le 1-\delta \quad \forall (i,j)\in H^{(\ell)},
\]
for some $0<\varepsilon<\delta<1$.
Then
\[
\sum_{(i,j)\in E^{(\ell)}} \psi_\gamma\!\left(p_{ij}^{(\ell)}\right)
\le
|E^{(\ell)}|(\gamma+1)\varepsilon^{\gamma+1},
\]
and
\[
\sum_{(i,j)\in H^{(\ell)}} \psi_\gamma\!\left(p_{ij}^{(\ell)}\right)
\ge
|H^{(\ell)}|\delta^{\gamma+1}.
\]
Therefore,
\[
\frac{
\sum_{(i,j)\in E^{(\ell)}} \psi_\gamma\!\left(p_{ij}^{(\ell)}\right)
}{
\sum_{(i,j)\in H^{(\ell)}} \psi_\gamma\!\left(p_{ij}^{(\ell)}\right)
}
\le
\frac{|E^{(\ell)}|}{|H^{(\ell)}|}
(\gamma+1)\left(\frac{\varepsilon}{\delta}\right)^{\gamma+1}.
\]
Hence Fair MPO suppresses the optimization influence of overrepresented easy pairs and mitigates imbalance caused by easy-majority dominance.

\noindent
\textbf{Proof.} 
Since
\[
\mathcal{L}_\gamma(p)=-(1-p)^\gamma\log p,
\]
we have
\[
\frac{\partial \mathcal{L}_\gamma}{\partial z}
=
(1-p)^\gamma\big(\gamma p\log p + p - 1\big),
\]
and thus
\[
\psi_\gamma(p)
=
(1-p)^\gamma\big((1-p)-\gamma p\log p\big).
\]
If $p\ge 1-\varepsilon$, then $1-p\le \varepsilon$ and
\[
-\log p \le \frac{1-p}{p},
\]
so
\[
-\gamma p\log p \le \gamma(1-p).
\]
Therefore
\[
\psi_\gamma(p)
\le
(1-p)^\gamma\big((1-p)+\gamma(1-p)\big)
=
(\gamma+1)(1-p)^{\gamma+1}
\le
(\gamma+1)\varepsilon^{\gamma+1}.
\]
Summing over $E^{(\ell)}$ yields the first bound.

If $p\le 1-\delta$, then $1-p\ge \delta$, and since $-\gamma p\log p\ge 0$,
\[
\psi_\gamma(p)
\ge
(1-p)^{\gamma+1}
\ge
\delta^{\gamma+1}.
\]
Summing over $H^{(\ell)}$ yields the second bound. Taking the ratio proves the claim.

\section{Experiments}

\subsection{Implementation.} 
We adopt the implementation of \cite{li2025flow} in our experiments.
Similar to \cite{li2025flow}, we adopt Qwen2.5-7B-Instruct  for  \textit{Action Planner}, \textit{Tool Executor}, \textit{Executive Verifier}, and \textit{Solution Generator}
Only the \textit{Action Planner} is trainable, while the others remain fixed during training.
Our framework is equipped with five interactive tools:
(1) \textbf{\textit{Base Generator}} uses Qwen2.5-7B-Instruct, which serves as the default reasoning engine when the planner chooses not to invoke an external tool; 
(2) \textbf{\textit{Python Coder}} generates and executes Python scripts for a given query and returns the execution results;
(3) \textbf{\textit{Google Search}} retrieves web search results and returns a summary of the top-$K$ retrieved contents;
(4) \textbf{\textit{Wikipedia Search}} retrieves relevant Wikipedia articles for a given query and returns a summarized response;
and (5) \textbf{\textit{Web Search}} summarizes the content of a given web page.
Following standard protocols \cite{li2025flow, jin2025search}, we use a learning rate of $1 \times 10^{-6}$.
The \textit{Action Planner} generates actions with a sampling temperature of $0.5$ to balance exploration and exploitation.
To stabilize training and prevent policy collapse \cite{jin2025search, li2025flow}, we set a KL-divergence penalty $\beta = 0.001$.
The maximum output length of the planner is set to 2048 tokens to allow sufficient exploration during rollout generation.
We use a batch size of 12 and sample $K=4$ rollouts for each input.
It should be noted that if all rollouts obtain the same ranking score $S(\tau)$, we sample an additional $K$ rollouts and recompute the ranking over all $2K$ candidates. If the scores remain tied, we exclude the sample from preference construction for the current update to avoid introducing arbitrary or noisy preference labels.
We set the maximum number of turns in each rollout at $3$ for training and $10$ for evaluation.
The reward function is provided by an GPT-4o (LLM-as-a-judge).
For stability, the LLM engines used within the tools operate with a temperature of 0.0, ensuring deterministic outputs.
The full training process is conducted on 12 NVIDIA L40S GPUs.
We adopt the agent prompts and memory update mechanism from \cite{li2025flow}.
We set the focal hyper-parameter $\gamma$ to $2.0$, and weighting parameters $\lambda_{\mathrm{DPO}}$ and $\lambda_{\mathrm{\Phi-MPO}}$ to $1.0$.

\subsection{Additional Ablation Study}

\begin{wraptable}[6]{r}{0.5\linewidth}
\centering
\vspace{-5mm}
\caption{Effectiveness of Weighting Parameters (Accuracy).}
\label{tab:abl-weighting}
\setlength{\tabcolsep}{2pt}
\vspace{-2mm}
\resizebox{1.0\linewidth}{!}{
\begin{tabular}{cc|cccc}
\hline
$\lambda_{\mathrm{DPO}}$ & $\lambda_{\mathrm{\Phi-MPO}}$    & Bamboogle & 2Wiki & GAIA & AIME24 \\
\hline
0.5 & 1.0 & 72.8          & 73.0          & 33.1          & 45.0    \\
1.0 & 1.0 & \textbf{81.6} & \textbf{81.5} & \textbf{41.7} & \textbf{50.0} \\
1.0 & 0.5 & 67.2          & 67.0          & 26.8          & 33.3          \\
\hline
\end{tabular}
}
\end{wraptable}

\textbf{Effectiveness of Weighting Parameters.}
Table~\ref{tab:abl-weighting} studies the impact of the weighting parameters $\lambda_{\mathrm{DPO}}$ and $\lambda_{\mathrm{\Phi-MPO}}$, which balance the trajectory-level DPO loss and the Fair MPO in our framework.
As shown in Table~\ref{tab:abl-weighting}, using equal weights for the two objectives yields the best performance on all benchmarks, achieving 81.6\% on Bamboogle, 81.5\% on 2Wiki, 41.7\% on GAIA, and 50.0\% on AIME24.
Reducing the trajectory-level weight to $\lambda_{\mathrm{DPO}}=0.5$ leads to a moderate drop in performance, showing that the global trajectory-level preference signal remains important for long-horizon reasoning.
However, reducing the multi-level weight to $\lambda_{\mathrm{\Phi-MPO}}=0.5$ causes a much larger degradation across all benchmarks, with the performance dropping to 67.2\% on Bamboogle, 67.0\% on 2Wiki, 26.8\% on GAIA, and 33.3\% on AIME24.
These results indicate that the fair multi-level objective plays a more critical role in our framework, since it provides fine-grained supervision over intermediate reasoning steps and is essential for effective agentic learning.

\end{document}